\documentclass[runningheads]{llncs}
\usepackage[T1]{fontenc}
\usepackage{graphicx}
\usepackage{amssymb}
\usepackage{graphicx}
\usepackage{amsfonts}
\usepackage{amsmath,bm}
\usepackage{bbm}
\usepackage{multirow}
\usepackage{makecell}
\usepackage{color}
\usepackage{algorithm}
\usepackage{algpseudocode}
\usepackage[misc]{ifsym}
\begin{document}
\title{New Lower Bound and Upper Bounds on the Regret for Online Sparse Linear Regression
\thanks{The authors are listed in alphabetical order.
This work was supported by the Fundamental Research Funds for the Central Universities.
This work was accepted by IJTCS-FAW 2026.}}
\titlerunning{New Lower Bound and Upper Bounds on the Regret for OSLR}

\author{Xiaofeng Cao\inst{2} \and Junfan Li\inst{1}(\Letter) \and 
Langzhang Liang\inst{3} \and
Mingwei Xu\inst{4} \and
Xiao Zhang\inst{5}}

\authorrunning{Cao et al.}

\institute{School of Computing and Artificial Intelligence, Shanghai University of Finance and Economics, Shanghai, China\\
\email{lijunfan@mail.sufe.edu.cn} \and
School of Computer Science and Technology, Tongji University, Shanghai, China\\
\email{xiaofeng.cao.uts@gmail.com} \and
AI$^3$ Institute, Fudan University and Shanghai Innovation Institute
\email{lazylzliang@gmail.com} \and
School of Artificial Intelligence, Jilin University, Changchun, China \\
\email{mingweixu20@outlook.com}\and
Gaoling School of Artificial Intelligence, Renmin University of China, Beijing, China\\
\email{zhangx89@ruc.edu.cn}}

\maketitle              
\begin{abstract}
  We study online sparse linear regression (OSLR)
  where any algorithm is restricted to accessing only $b$ out of $d$ attributes per instance for prediction
  and $b_0\geq 0$ additional attributes after prediction,
  which was proved to be NP-hard.
  Previous work focused on designing computationally efficient algorithms under regularity assumptions,
  but did not characterize its information theoretic complexity.
  In this work,
  we give the first lower bound on the minimax regret of OSLR and
  design algorithms with better upper bounds without regularity assumptions.
  We characterize how minimax regret scales with problem-dependent parameters,
  capturing the information theoretic complexity of OSLR.
\end{abstract}

\section{Introduction}

    For most online prediction tasks,
    learning algorithms are assumed to observe all of the attributes of an instance ${\bf x}_t\in\mathbb{R}^d$
    at each round $t$.
    However,
    the assumption is hard to be satisfied in some real-world scenarios due to various constraints,
    such as computational constraint, human labors and privacy constraint
    \cite{Cesa-Bianchi2010Efficient,Hazan2012Linear,Jain2012Differentially,Zolghadr2013Online,Murata2018Sample}.
    For instance,
    in the task of medical diagnosis of a disease \cite{Cesa-Bianchi2010Efficient},
    ${\bf x}_t$ contains the results of a large number of medical tests.
    However, many patients can only pay the cost for several medical tests.
    Thus ${\bf x}_t$ is sparse.
    Another example is personalized recommendation \cite{Jain2012Differentially}.
    Due to privacy constraint,
    search engines are not allowed to use sensitive attributes of users,
    such as gender, age, job and so on,
    ${\bf x}_t$ is also sparse.
    Algorithms typically make predictions using limited attributes.

    Previous work formulated online learning problems with limited access to attributes
    as online sparse linear regression (OSLR)
    \cite{Kale2014Open,Foster2016Online}.
    At each round $t$,
    an adversary gives an instance ${\bf x}_t$ to a learner.
    The learner uses $b$, $b<d$, attributes of ${\bf x}_t$ at most,
    and outputs a prediction $\hat{y}_t$.
    Then the adversary gives the true output $y_t$.
    The learner suffers a square loss $(\hat{y}_t-y_t)^2$.
    The learner aims to minimize the cumulative losses.
    We usually compare the cumulative losses of the learner
    with any $b$-sparse competitor with bounded norm,
    and define the regret as follows,
    \begin{equation}
    \label{eq:AAAI25:definition_regret}
    \begin{split}
        \forall {\bf w}\in\mathcal{W}_b,\quad
        \mathrm{Reg}({\bf w})=&\sum^T_{t=1}\left(\hat{y}_t-y_t\right)^2
        -\sum^T_{t=1}\left({\bf w}^\top{\bf x}_t-y_t\right)^2,\\
        \mathcal{W}_b=&\left\{{\bf w}\in\mathbb{R}^d:\Vert{\bf w}\Vert_0\leq b,\Vert{\bf w}\Vert_2\leq U\right\},
    \end{split}
    \end{equation}
    in which $U>0$ is a constant.
    There is also a relaxation of OSLR,
    in which algorithms are allowed to observe additional $b_0=O(b\ln{d})$ attributes
    after each prediction.
    For simplicity,
    we call the two problems $(d,b,b_0)$-OSLR.

    Previous work mainly focused on developing algorithms that
    run in time $O(\mathrm{poly}(T,d,b,b_0))$ per-iteration,
    and enjoy a regret of $O(\mathrm{poly}(d,b,b_0)\cdot o(T))$.
    It has been proved that $(d,b,b_0)$-OSLR is NP-hard \cite{Foster2016Online}.
    Specifically,
    there is no algorithm running in time $O(\mathrm{poly}(T,d,b,b_0))$ per-iteration
    and enjoying a regret of $O(\mathrm{poly}(d,b,b_0)\cdot T^{1-\delta})$
    for any constant $\delta>0$ unless $\mathbf{NP}\subseteq \mathbf{BPP}$.
    To address this computational barrier,
    it is necessary to restrict this problem with regularity assumptions.
    If the data matrix $\mathbf{X}=[\mathbf{x}_1,\ldots,\mathbf{x}_T]$
    satisfies the restricted isometry property (RIP) \cite{Candes2005Decoding},
    then there is an algorithm for $(d,b,b_0)$-OSLR with $b_0\geq 2$,
    which runs in time $O(\mathrm{poly}(d))$ per-iteration and enjoys a high-probability regret of
    $O(\mathrm{poly}(d,b,b_0)\ln^2(Td))$.
    If the expected covariance matrix $\mathbb{E}[\mathbf{x}\mathbf{x}^\top]$ is non-singular,
    a stronger one than RIP,
    then there are algorithms for $(d,b,b_0)$-OSLR with $b_0\leq 1$
    and $b_0\geq 2$ that run in time $O(d)$ per-iteration
    and enjoy an expected regret of $O(\mathrm{poly}(d,k)\sqrt{T})$ \cite{Ito2017Efficient}.
    Suppose $\mathbf{X}$ satisfies the compatibility condition \cite{Sara2007The},
    a weaker one than RIP,
    then there are algorithms for $(d,b,b_0)$-OSLR with $b_0\leq 1$
    and $b_0\geq 3$ that run in time $O(\mathrm{poly}(d))$
    and separately enjoy a high-probability regret of $O(\mathrm{poly}(d,b)\sqrt{T})$
    and $O(\mathrm{poly}(d,b,b_0)\ln^2(Td))$ \cite{Li2025A}.

    However, the information theoretic complexity of $(d,b,b_0)$-OSLR remains unexplored.
    Specifically,
    without regularity assumptions,
    the minimax regret remains unknown.
    Prior work only gave two upper bounds on the minimax regret.
    In the case of $b_0\geq 2$,
    there is an algorithm that enjoys an expected regret of $O(\frac{d^2}{b^2_0}\sqrt{T\ln{N}})$
    where $N=\binom{d}{b}$ \cite{Foster2016Online}.
    In the case of $b_0\leq 1$,
    there is an algorithm that enjoys an expected regret of
    $O\left(N^{\frac{1-q}{3-2q}}T^{\frac{2-q}{3-2q}}\right),q\in(0,1)$ \cite{Ito2018Online}.
    It seems that we can not obtain a $O(\sqrt{T})$ regret in the case of $b_0\leq 1$.
    As the lower bound is unknown,
    it is unknown that whether the two algorithms are optimal.

\subsection{Main Results}

    In this work,
    we aim to characterize the information theoretic complexity of $(d,b,b_0)$-OSLR.
    We first prove a $\Omega(\sqrt{T})$ lower bound on the minimax regret
    for all $b_0\geq 0$ and $b_0<d-b$.
    Then we separately propose a new algorithm
    for the cases $b_0\geq 2$ and $b_0\leq 1$.
    All achieve a $O(\sqrt{T})$ upper bound,
    yet differ in problem-dependent constants.
    Our main results are summarized in Table \ref{tab:IJTCS-FAW26:main_results_summarization}.

    \begin{itemize}
      \item We give the first $\Omega\left(\sqrt{dT/(b(b+b_0))}\right)$ lower bound
      on the minimax regret of $(d,b,b_0)$-OSLR.
            This lower bound shows how problem-dependent parameters decide
            the information theoretic complexity of $(d,b,b_0)$-OSLR.
      \item In the case of $b_0\leq 1$,
            we propose a new algorithm that enjoys a $O(\sqrt{NT}\ln{T})$ regret bound,
            improving previous $O\left(N^{\frac{1-q}{3-2q}}T^{\frac{2-q}{3-2q}}\right),q\in(0,1)$
            regret bound \cite{Ito2018Online} in terms of the dependence on $T$.
      \item In the case of $b_0\geq 2$,
            we propose a new algorithm that enjoys a $O(\frac{d}{b_0}\sqrt{T\ln{N}})$ regret bound,
            improving previous $O(\frac{d^2}{b^2_0}\sqrt{T\ln{N}})$
            regret bound \cite{Foster2016Online}.
    \end{itemize}

    \begin{table}[!t]
        \caption{Regret Lower Bound and Upper Bounds for $(d,b,b_0)$-OSLR. $N=\binom{d}{b}$}
        \centering
        \label{tab:IJTCS-FAW26:main_results_summarization}
        \begin{tabular}{l|l|l}
        \hline
        $(d,b,b_0)$-OSLR &  Upper Bound & Lower Bound\\
        \hline
        \multirow{2}{*}{$b_0\geq 2$}    &  {\color{blue}$O\left(\sqrt{\frac{(d-b)(d-b-1)}{b_0(b_0-1)}}\sqrt{T\ln{N}}\right)$ (Ours)} & \multirow{5}{*}{{\color{blue}$\Omega\left(\frac{1}{\sqrt{b}}\sqrt{\frac{dT}{b+b_0}}\right)$ (Ours)}}\\
        &  $O\left(\frac{d^2}{b^2_0}\sqrt{T\ln{N}}\right)$ \cite{Foster2016Online} & \\
        \cline{1-2}
        \multirow{2}{*}{$b_0\in\{0,1\}$}          &  {\color{blue}$O(\sqrt{NT}\ln{T})$ (Ours)} & \\
        &  $O\left(N^{\frac{1-q}{3-2q}}T^{\frac{2-q}{3-2q}}\right),q\in(0,1)$ \cite{Ito2018Online} & \\
        \hline
        \end{tabular}
    \end{table}

    Our work for the first time characterizes the information theoretic complexity of $(d,b,b_0)$-OSLR.
    Although there are large gaps between the lower bound and the upper bounds
    in terms of problem-dependent constants,
    including $d,b,b_0$,
    our work reveals that the $O(\sqrt{T})$ factor is optimal,
    and the number of additional observed attributes $b_0$
    may decide the complexity of the problem.

\subsection{Technical Challenges}

    At each round $t$,
    suppose $S_t\subseteq [d]$ denotes the set of indices of attributes that an algorithm chooses to observe,
    and $\mathbf{w}_t\in\mathbb{R}^b$ denotes its model parameter.
    The regret of the algorithm depends on the observations $\mathbf{x}_t$,
    the true output $y_t$, $\mathbf{w}_t$ and $S_t$,
    making it challenging to analyze the lower bound.
    Our main techniques aim to reduce OSLR to a more simple problem.

    To decouple $\mathbf{x}_t$ and $y_t$,
    let $x_{t,i}=\sigma_tz_{t,i}$ for all $i\in[d]$ and $y_t=\gamma \sigma_t$
    where both $\sigma_t$ and $z_{t,i}$ are Bernoulli random variables,
    and $\gamma$ is a constant depending on $U$.
    Given $(\mathbf{x}_t,y_t)$,
    for any $\mathbf{w}\in\mathbb{R}^b$ and $S\subseteq [d]$,
    the square loss only depends on $\mathbf{z}_t$.
     \begin{align*}
        \left(\left\langle\mathbf{w},\mathbf{x}_t[S]\right\rangle-y_t\right)^2
        =&\mathbf{w}^\top\mathbf{z}_t[S](\mathbf{z}_t[S])^\top\mathbf{w}
        -2\gamma\cdot\left\langle\mathbf{w},\mathbf{z}_t[S]\right\rangle+\gamma^2,
    \end{align*}
    Thus the regret only depends on $\mathbf{z}_t$, $\mathbf{w}_t$ and $S_t$.

    To decouple $S_t$ and $\mathbf{w}_t$,
    we construct a sequence of virtual model parameters denoted by
    $\mathbf{w}_{S_t,\ast}\in\mathbb{R}^b$ such that
    $\mathbf{w}_{S_t,\ast}$ is the optimal parameter if $\mathbf{x}_t[S_t]$ are observed.
    Thus the loss of $\mathbf{w}_{S_t,\ast}$ is no more than that of $\mathbf{w}_t$,
    i.e.,
    $$
        \left(\langle \mathbf{w}_{S_t,\ast},\mathbf{x}_t[S_t]\rangle-y_t\right)^2
        \leq \left(\langle \mathbf{w}_t,\mathbf{x}_t[S_t]\rangle-y_t\right)^2.
    $$
    In this way,
    we only need to analyze the regret lower bound of $\{\mathbf{w}_{S_t,\ast}\}^T_{t=1}$,
    which depends on the index sets $\{S_t\}^T_{t=1}$
    and $\{\mathbf{z}_t\}^T_{t=1}$,
    reducing OSLR to the problem of identifying optimal attributes of $\mathbf{z}_t$ given restricted observations.


\section{Problem Setting}

    Let $\left\{\left(\mathbf{x}_t,y_t\right)\right\}_{t\in[T]}$ be a sequence of examples
    where $\mathbf{x}_t\in\mathcal{X}=\{{\bf x}\in\mathbb{R}^d\vert\Vert {\bf x}\Vert_2\leq X\}$ is an instance,
    $\vert y_t\vert \leq Y$ is the output,
    and $[T] = \{1,\ldots,T\}$.
    For any vector $\mathbf{a}\in\mathbb{R}^d$,
    denote by $a_i$ its $i$-th coordinate.
    For any matrix $\mathbf{A}\in\mathbb{R}^{m\times n}$,
    denote by $A[i,j]$ the element on the $i$-th row and $j$-th column.
    For any $S=\{s_1,s_2,...,s_b\}\subset[d]$,
    we define ${\bf a}[S]:=(a_{s_1},a_{s_2},...,a_{s_b})$
    and
    ${\bf A}[S,S]\in\mathbb{R}^{\vert S\vert\times \vert S\vert}$
    the sub-matrix formed by the rows and columns indexed by $S$.
    Let $N=\binom{d}{b}$ and $\Delta_N$ be the $N-1$ dimensional simplex.
    For a convex set $\Omega$
    and a convex function $\psi$ defined on $\Omega$,
    the Bregman divergence denoted by $\mathcal{B}_{\psi}(\cdot,\cdot)$,
    is defined as follows
    \begin{align*}
        \forall {\bf u},{\bf v}\in \Omega,~\mathcal{B}_{\psi}(\mathbf{u},\mathbf{v})
        =\psi(\mathbf{u})-\psi(\mathbf{v})-\langle\nabla \psi(\mathbf{v}),\mathbf{u}-\mathbf{v}\rangle.
    \end{align*}
    Let $\mathrm{Supp}(\mathbf{w})=\{i\in[d]:w_i\neq 0\}$ be the support set of a vector $\mathbf{w}$.


\section{Lower Bound on the Minimax Regret}


    \begin{theorem}
    \label{thm:IJTSC-FAW26:lower_bound_regret:OSLR}
        Let $U>0$, $\beta\geq 12$, $0<\alpha<1-\frac{3}{\beta}$ and $\gamma=\min\{1,U/\sqrt{3}\}$.
        Suppose $T\geq \frac{5}{2}d$ and $b\leq \alpha d$.
        For any $d>b_0+b\geq b\geq 1$,
        the regret of any algorithm for $(d,b,b_0)$-OSLR satisfies
        \begin{align*}
            \sup_{({\bf x}_t,y_t),t\in[T]}\mathbb{E}\left[\max_{{\bf w}\in\mathcal{W}_b}
            \mathrm{Reg}({\bf w})\right]\geq
            \left(1-\frac{3}{\beta}-\alpha\right)
            \frac{3\sqrt{10}\gamma^2}{5\beta}\cdot\sqrt{\frac{d}{b(b+b_0)}T},
        \end{align*}
        where the expectation is taken over the internal randomness of algorithm.
    \end{theorem}

    Theorem \ref{thm:IJTSC-FAW26:lower_bound_regret:OSLR} shows
    a $\Omega(\sqrt{T})$ lower bound for any $b_0\geq 0$.
    In the case of $b_0\geq 2$,
    there is an algorithm \cite{Foster2016Online}
    that enjoys a $O(\frac{d^2}{b^2_0}\sqrt{T\ln{N}})$ regret bound,
    demonstrating that the $\Theta(\sqrt{T})$ factor is optimal.
    However, there is a large gap in terms of problem-dependent constants,
    including $d$, $b$, $b_0$.
    In the case of $b_0\leq 1$,
    there is an algorithm enjoys a
    $O\left(N^{\frac{1-q}{3-2q}}T^{\frac{2-q}{3-2q}}\right),q\in(0,1)$ regret bound \cite{Ito2018Online}.
    There are large gaps in terms of problem-dependent constants and $T$.
    It remains unknown what the optimal order in $T$ is.
    In the subsequent sections,
    we separately design a new algorithm for the cases $b_0\geq 2$ and $b_0\leq 1$,
    demonstrating that the the $\Theta(\sqrt{T})$ factor is optimal for the two cases
    and improving the problem-dependent constant factor in the case of $b_0\geq 2$.

\begin{proof}

\noindent\textbf{Step 1: Constructing Hard Instances}

    For simplicity, let $b'=b+b_0$.
    Suppose $d>b'\geq b\geq 1$.
    We sample an attribute from $[d]$ uniformly, denoted by $h$,
    then construct $({\bf x}^h_t,y_t),t=1,...,T$ as follows.
    \begin{equation}
    \label{eq:AAAI25_supp:constructing_hard_examples}
    \left\{
    \begin{split}
        x^h_{t,i}=&\sigma_t\cdot z^h_{t,i},\quad \forall i=1,\ldots,d,\\
        y_t=&\gamma\cdot\sigma_t,\\
        \mathbb{P}[\sigma_t=1]=&\mathbb{P}[\sigma_t=-1]=\frac{1}{2},\\
        \mathbb{P}[z^h_{t,i}=1]=&\frac{1+\rho}{2},\qquad\quad \mathbb{P}[z^h_{t,i}=-1]=\frac{1-\rho}{2},\quad \forall i\neq h,\\
        \mathbb{P}[z^h_{t,h}=1]=&\frac{1+\rho+\epsilon}{2},~\quad \mathbb{P}[z^h_{t,h}=-1]=\frac{1-\rho-\epsilon}{2},
    \end{split}
    \right.
    \end{equation}
    where
    $$
        \rho=\frac{1}{2\sqrt{b}},\quad \gamma=\min\{1,U/\sqrt{3}\},\quad 0<\epsilon\leq\frac{\rho}{3}.
    $$
    Multiplying by the random variable $\sigma_t$ on $\gamma$ aims to deteriorate
    the deterministic learner who always outputs the prediction $\gamma$.
    Let $S\subseteq [d]$ and $\vert S\vert=b$ be a support set.
    Let $S^{\neg}=[d]\setminus S$.
    For each $i\in[d]$,
    let $\mu_i:=\mathbb{E}\left[z^h_{t,i}\right]$ and
    $\mathbf{G}$ be the expected covariance matrix of $\mathbf{z}^h_t$ given by
    $$
        \mathbf{G}=\mathbb{E}\left[\mathbf{z}^h_t\left(\mathbf{z}^h_t\right)^\top\right]
        =\mathbf{D}+{\bm \mu}{\bm \mu}^\top,\quad
        \mathbf{D}=\mathrm{diag}(\mathbf{1}-{\bm \mu}\odot{\bm \mu}),
    $$
    where $\odot$ is element-wise multiplication.

    Given a support set $S\subseteq [d]$ and a model parameter $\mathbf{w}$,
    at any round $t\in[T]$,
    by (\ref{eq:AAAI25_supp:constructing_hard_examples}),
    there is an important observation
    \begin{align*}
        \left(\left\langle\mathbf{w}[S],\mathbf{x}^h_t[S]\right\rangle-y_t\right)^2
        =&\underbrace{\mathbf{w}[S]^\top\mathbf{z}^h_t[S](\mathbf{z}^h_t[S])^\top\mathbf{w}[S]
        -2\gamma\cdot\left\langle\mathbf{w}[S],\mathbf{z}^h_t[S]\right\rangle+\gamma^2}_
        {:=\bar{\ell}(\mathbf{w},S,\mathbf{z}^h_t)}.
    \end{align*}
    Thus the loss function only depends on $S$ and $\mathbf{z}^h_t$.
    In other words,
    given (\ref{eq:AAAI25_supp:constructing_hard_examples}),
    OSLR is equivalent to a new partial information online learning problem,
    which we call it ``online sparse linear learning (OSLL).
    \begin{definition}[OSLL]
        At each round $t$,
        the environment first assigns an observation $\mathbf{z}^h_t$ following
        (\ref{eq:AAAI25_supp:constructing_hard_examples}).
        A learner maintains a model parameter $\mathbf{w}_t$,
        selects a subset $S_t\subseteq[d]$,
        and observes $\mathbf{z}^h_t[S_t]$.
        Then learner suffers a loss $\bar{\ell}(\mathbf{w}_t,S_t,\mathbf{z}^h_t)$
        and observes additional $b_0$ elements of $\mathbf{z}^h_t$.
    \end{definition}

    Given (\ref{eq:AAAI25_supp:constructing_hard_examples}),
    from the perspective of regret minimization,
    OSLR is equivalent to OSLL.
    For any parameters $\{\mathbf{w}_t\}^T_{t=1}$ and support sets $\{S_t\}^T_{t=1}$,
    it must be
    $$
        \sum^T_{t=1}
        [\ell(\mathbf{w}_t,S_t,\mathbf{x}^h_t,y_t)-\ell(\mathbf{w},S,\mathbf{x}^h_t,y_t)]
        =\sum^T_{t=1}
        \left[\bar{\ell}(\mathbf{w}_t,S_t,\mathbf{z}_t)-\bar{\ell}(\mathbf{w},S,\mathbf{z}_t)\right].
    $$
    The regret on OSLR is equal to that on OSLL.
    Therefore,
    we only need to analyze the lower bound on the regret for OSLL.

    In OSLL,
    at each round $t$,
    any algorithm observes $b$ elements of $\mathbf{z}^h_t$, and computes the loss.
    Taking expectation w.r.t. $\mathbf{z}^h_t$ gives
    \begin{align*}
        \mathbb{E}\left[\bar{\ell}(\mathbf{w},S,\mathbf{z}^h_t)\right]
        =\mathbf{w}[S]^\top \mathbf{G}[S]\mathbf{w}[S]
        -2\gamma\cdot\left\langle\mathbf{w}[S],{\bm \mu}[S]\right\rangle+\gamma^2.
    \end{align*}
    Let $\mathbf{w}_{S,\ast}[S]$ be the optimal parameter tailored to a support set $S$,
    which is
    $$
        \mathbf{w}_{S,\ast}[S]=\gamma\cdot \mathbf{G}[S]^{-1}{\bm \mu}[S],\quad
        \mathbf{w}_{S,\ast}[S^{\neg}]=\mathbf{0}.
    $$
    The expected loss of $\mathbf{w}_{\ast}[S]$ is
    \begin{align*}
        \mathbb{E}\left[\bar{\ell}(\mathbf{w}_{S,\ast},S,\mathbf{z}^h_t)\right]
        =\gamma^2\cdot\frac{1}{1+{\bm\mu}[S]^\top \mathbf{D}[S]^{-1}{\bm\mu}[S]}
        =\gamma^2\cdot\frac{1}{1+\sum_{i\in S}\frac{\mu^2_i}{1-\mu^2_i}},
    \end{align*}
    which only depends on $S$ and the expectation of $z^h_{t,i}$, $i\in S$.
    To minimize the expected square loss,
    the optimal support set $S_{\ast}$ must satisfy
    $$
        S_{\ast}=\{h\}\cup\{O\subset [d]\setminus\{h\}:\vert O\vert=b-1\}.
    $$
    Let $\mathbf{w}_{\ast}$ be the optimal model parameter which satisfies
    \begin{align*}
        \mathbf{w}_{\ast}[S_{\ast}]=&\gamma\cdot \mathbf{G}[S_{\ast}]^{-1}{\bm \mu}[S_{\ast}],\quad
        \mathbf{w}_{\ast}[S^{\neg}_{\ast}]=\mathbf{0},\\
        \mathbb{E}\left[\bar{\ell}(\mathbf{w}_{\ast},S_{\ast},\mathbf{z}^h_t)\right]=&
        \gamma^2\cdot\frac{1}{1+\frac{(\rho+\epsilon)^2}{1-(\rho+\epsilon)^2}+(b-1)\cdot\frac{\rho^2}{1-\rho^2}},
    \end{align*}
    in which we define $\mathbf{w}_{\ast}[S^{\neg}_{\ast}]=\mathbf{0}$
    for minimizing its norm.
    It can be verified that
    $$
        \Vert \mathbf{w}_{\ast}\Vert_2=
        \Vert \mathbf{w}_{\ast}[S_{\ast}]\Vert_2\leq \sqrt{3}\gamma\leq U.
    $$
    Thus $\mathbf{w}_{\ast}\in\mathcal{W}_b$.

\noindent\textbf{Step 2: Decoupling the Influence of $\mathbf{w}_t$ and $S_t$ on the Regret}

    The regret depends on the scale of $w_{t,i}$ for all $i\in S_t$,
    and the support set $S_t$ itself,
    making it hard to analysis the lower bound.
    To decouple $\mathbf{w}_t$ and $S_t$,
    we construct a virtual sequence of parameters $\{\mathbf{w}_{S_t,\ast}\}^T_{t=1}$ where
    $$
        \mathbf{w}_{S_t,\ast}[S_t]=\gamma\cdot \mathbf{G}[S_t]^{-1}{\bm \mu}[S_t],\quad
        \mathbf{w}_{S_t,\ast}[S^{\neg}_t]=\mathbf{0}.
    $$
    As $\mathbf{w}_{S_t,\ast}$ is the optimal parameter tailed to $S_t$,
    we have
    \begin{align*}
        \sum^T_{t=1}\mathbb{E}\left[\bar{\ell}(\mathbf{w}_t,S_t,\mathbf{z}^h_t)
        -\bar{\ell}(\mathbf{w}_{\ast},S_{\ast},\mathbf{z}^h_t)\right]
        \geq\sum^T_{t=1}\mathbb{E}\left[\bar{\ell}(\mathbf{w}_{S_t,\ast},S_t,\mathbf{z}^h_t)-
        \bar{\ell}(\mathbf{w}_{\ast},S_{\ast},\mathbf{z}^h_t)\right].
    \end{align*}

    Thus we only need to analyze the regret of $\{\mathbf{w}_{S_t,\ast}\}^T_{t=1}$
    that only depends on $S_t$.
    Specifically,
    if $h\neq S_t$, then
    the expected per-round regret is given by
    \begin{align}
        &\mathbb{E}\left[\bar{\ell}(\mathbf{w}_{S_t,\ast},S_t,\mathbf{z}^h_t)\right]
        -\mathbb{E}\left[\bar{\ell}(\mathbf{w}_{\ast},S_{\ast},\mathbf{z}^h_t)\right]\nonumber\\
        =&\gamma^2\cdot\frac{1}{1+b\cdot\frac{\rho^2}{1-\rho^2}}-
        \gamma^2\cdot\frac{1}{1+\frac{(\rho+\epsilon)^2}{1-(\rho+\epsilon)^2}+(b-1)\cdot\frac{\rho^2}{1-\rho^2}}
        \geq\frac{3}{5\sqrt{b}}\gamma^2\cdot\epsilon,
        \label{eq:ICLR2027:lower_bound:per-round_regret}
    \end{align}
    where the last inequality comes from
    $\rho^2\leq\frac{1}{4b}$ and $\frac{\rho^2}{1-\rho^2}\leq \frac{1}{3b}$.
    Summing over $t=1,2,\ldots,T$,
    the regret can be lower bounded by
    \begin{align}
        &\mathbb{E}_{\mathbf{z}^h_t,t\in[T]}\left[\sum^T_{t=1}\left[\bar{\ell}(\mathbf{w}_{S_t,\ast},S_t,\mathbf{z}^h_t)\right]
        -\left[\bar{\ell}(\mathbf{w}_{\ast},S_{\ast},\mathbf{z}^h_t)\right]\right]\nonumber\\
        =&\sum^T_{t=1}\mathbb{E}_{\mathbf{z}^h_{\tau},\tau\in[t-1]}\cdot
        \mathbb{E}_{\mathbf{z}^h_t\vert\mathbf{z}^h_{\tau},\tau\in[t-1]}\left[\bar{\ell}(\mathbf{w}_{S_t,\ast},S_t,\mathbf{z}^h_t)
        -\bar{\ell}(\mathbf{w}_{\ast},S_{\ast},\mathbf{z}^h_t)\right]\nonumber\\
        \geq&\sum^T_{t=1}\mathbb{E}_{\mathbf{z}^h_{\tau},\tau\in[t-1]}\cdot
        \mathbb{E}_{\mathbf{z}^h_t\vert\mathbf{z}^h_{\tau},\tau\in[t-1]}\left[\bar{\ell}(\mathbf{w}_{S_t,\ast},S_t,\mathbf{z}^h_t)
        -\bar{\ell}(\mathbf{w}_{\ast},S_{\ast},\mathbf{z}^h_t)\right]\cdot\mathbb{I}_{h\notin S_t}\nonumber\\
        \mathop{\geq}^{(\ref{eq:ICLR2027:lower_bound:per-round_regret})}&
        \sum^T_{t=1}\mathbb{E}_{\mathbf{z}^h_{\tau},\tau\in[t-1]}\cdot
        \mathbb{E}_{\mathbf{z}^h_t\vert\mathbf{z}^h_{\tau},\tau\in[t-1]}\left[\frac{3\gamma^2\cdot\epsilon}{5\sqrt{b}}\cdot
        \mathbb{I}_{h\notin S_t}\right]\nonumber\\
        =&\frac{3\gamma^2\epsilon}{5\sqrt{b}}\cdot \left[T-\mathbb{E}_h\left[\sum^T_{t=1}\mathbb{I}_{h\in S_t}\right]\right].
        \label{eq:ICLR2027:regret_to_number_of_selection_optimal_attributes}
    \end{align}
    For simplicity,
    let $N_h=\sum^T_{t=1}\mathbb{I}_{h\in S_t}$.
    Next we give an upper bound on $\mathbb{E}_h[N_h]$.

\noindent\textbf{Step 3: Analyzing $\mathbb{E}_h[N_h]$}

    At the end of each round $t$,
    the algorithm can observe additional $b_0$ attributes.
    Thus total number of observed attributes is $b'$.
    Denote by $O_t$ the index of observed attributes satisfying $\vert O_t\vert\leq b'$.
    Let $\mathbf{r}_t=\mathbf{z}^h_t[O_t]\in\{-1,1\}^{b'}$.
    It is obvious that
    at each round $t$,
    $\mathbb{I}_{h\in S_t}$ is a function of the historical observations $\{\mathbf{r}_\tau\}^{t-1}_{\tau=1}$.
    Let $\mathbf{P}_h$ be the distribution of $\mathbf{r}=(\mathbf{r}_1,\mathbf{r}_2,\ldots,\mathbf{r}_T)\in\{-1,1\}^{d'\times T}$
    under that the $h$-attribution follows the distribution in
    (\ref{eq:AAAI25_supp:constructing_hard_examples}).
    We also define a reference distribution $\mathbf{P}_u$
    that satisfies for any $i=1,\ldots,d$,
    $\mathbb{P}[z^h_{t,i}=1]=\frac{1+\rho}{2}$ and $\mathbb{P}[z^h_{t,i}=-1]=\frac{1-\rho}{2}$.

    The main idea of the proof is to show that
    the difference $\frac{1}{d}\sum^T_{t=1}\mathbb{E}_h[N_h]
    -\frac{1}{d}\sum^T_{t=1}\mathbb{E}_u[N_h]$ is small enough
    where $\mathbb{E}_h[\cdot]$ is the expectation under $\mathbf{P}_h$,
    and $\mathbb{E}_u[\cdot]$ is the expectation under $\mathbf{P}_u$.
    Given $\frac{1}{d}\sum^T_{t=1}\mathbb{E}_u[N_h]=\frac{bT}{d}$,
    we can derive a tight upper bound on $\frac{1}{d}\sum^T_{t=1}\mathbb{E}_h[N_h]$,
    yielding a lower bound on the regret.

        We rewrite $N_h$ as $N_h(\mathbf{r})$ where
        $\mathbf{r}\in\{-1,1\}^{b'\times T}$.
        \begin{align*}
            \mathbb{E}_h[N_h(\mathbf{r})]-&\mathbb{E}_u[N_h(\mathbf{r})]
            =\sum_{\mathbf{r}}N_h(\mathbf{r})\mathbf{P}_h[\mathbf{r}]-\sum_{\mathbf{r}}N_h(\mathbf{r})\mathbf{P}_u[\mathbf{r}]\\
            \leq&\sum_{\mathbf{r}:\mathbf{P}_h[\mathbf{r}]\geq\mathbf{P}_u[\mathbf{r}]}
            N_h(\mathbf{r})\mathbf{P}_h[\mathbf{r}]-\sum_{\mathbf{r}}N_h(\mathbf{r})\mathbf{P}_u[\mathbf{r}]\\
            \leq& T\sqrt{\Vert \mathbf{P}_u-\mathbf{P}_h\Vert^2_{\mathrm{TV}}}
            \leq\frac{T}{\sqrt{2}}\sqrt{\mathrm{KL}(\mathbf{P}_u,\mathbf{P}_h)}\quad
            \mathrm{by~Pinsker's~inequality},
        \end{align*}
        where $\Vert \mathbf{P}_u-\mathbf{P}_h\Vert_{\mathrm{TV}}$ is the total variation distance.
        Summing over $h=1,\ldots,d$,
        \begin{align}
            \sum^d_{h=1}\frac{\mathbb{E}_h[N_h]}{d}
            \leq&\sum^d_{h=1}
            \frac{\mathbb{E}_u[N_h]+T\sqrt{\frac{1}{2}\mathrm{KL}(\mathbf{P}_u,\mathbf{P}_h)}}{d}\nonumber\\
            \leq&\frac{Tb}{d}+\frac{T}{\sqrt{2d}}\sqrt{\sum^d_{h=1}\mathrm{KL}(\mathbf{P}_u,\mathbf{P}_h)}.
            \label{eq:ICLR2027:upper_bound_Nh_by_KL_divergence}
        \end{align}
        Next we give an upper bound on the KL divergence
        $\mathrm{KL}(\mathbf{P}_u,\mathbf{P}_h)$ by chain rule.
        Denote by $\mathbf{r}_{1:t}=(\mathbf{r}_1,\ldots,\mathbf{r}_t)$.
        Then
        \begin{align*}
            &\mathrm{KL}(\mathbf{P}_u,\mathbf{P}_h)=\sum^T_{t=1}\sum_{\mathbf{r}_{1:t-1}}\mathbf{P}_u(\mathbf{r}_{1:t-1})\cdot
            \left[\mathrm{KL}(\mathbf{P}_u(\mathbf{r}_t\vert \mathbf{r}_{1:t-1}),\mathbf{P}_h(\mathbf{r}_t
            \vert \mathbf{r}_{1:t-1}))\right]\\
            =&\sum^T_{t=1}\sum_{\mathbf{r}_{1:t-1}}\mathbf{P}_u(\mathbf{r}_{1:t-1})\cdot
            \mathbb{I}_{h\in O_t\vert \mathbf{r}_{1:t-1}}\cdot
            \left[\sum_{\mathbf{r}_t\in\{-1,1\}^b}\mathbf{P}_u(\mathbf{r}_t\vert \mathbf{r}_{1:t-1})\cdot
            \ln\frac{\mathbf{P}_u(\mathbf{r}_t\vert \mathbf{r}_{1:t-1})}{\mathbf{P}_h(\mathbf{r}_t\vert \mathbf{r}_{1:t-1})}\right]\\
            =&\sum^T_{t=1}\sum_{\mathbf{r}_{1:t-1}}\mathbf{P}_u(\mathbf{r}_{1:t-1})\cdot
            \mathbb{I}_{h\in O_t\vert \mathbf{r}_{1:t-1}}\cdot\\
            &\sum_{\mathbf{r}_t\in\{-1,1\}^b}\mathbf{P}_u(\mathbf{r}_t\vert \mathbf{r}_{1:t-1})\cdot
            \ln\frac{\prod^{b'}_{i=1}\frac{1+\rho(1-2\mathbb{I}_{r_{t,i}=-1})}{2}}
            {\prod_{i\neq h}\frac{1+\rho(1-2\mathbb{I}_{r_{t,i}=-1})}{2}
             \frac{1+(\rho+\epsilon)(1-2\mathbb{I}_{r_{t,h}=-1})}{2}}\\
            =&\sum^T_{t=1}\sum_{\mathbf{r}_{1:t-1}}\mathbf{P}_u(\mathbf{r}_{1:t-1})\cdot
            \mathbb{I}_{h\in O_t\vert \mathbf{r}_{1:t-1}}\cdot
            \left[\frac{1+\rho}{2}\cdot
            \ln\frac{1+\rho}{1+\rho+\epsilon}+
            \frac{1-\rho}{2}\cdot
            \ln\frac{1-\rho}{1-\rho-\epsilon}\right]\\
            \leq&\sum^T_{t=1}\sum_{\mathbf{r}_{1:t-1}}\mathbf{P}_u(\mathbf{r}_{1:t-1})\cdot
            \mathbb{I}_{h\in O_t\vert \mathbf{r}_{1:t-1}}\cdot
            \frac{(\frac{1+\rho}{2}-\frac{1+\rho+\epsilon}{2})^2}{\frac{1+\rho+\epsilon}{2}\frac{1-\rho-\epsilon}{2}}
            \leq\frac{9}{5}\epsilon^2\mathbb{E}_u\left[\sum^T_{t=1}\mathbb{I}_{h\in O_t}\right],
        \end{align*}
        in which we use the fact $\epsilon\leq \frac{1}{3}\rho$ and $\rho\leq \frac{1}{2}$.
        Substituting into (\ref{eq:ICLR2027:upper_bound_Nh_by_KL_divergence}) yields
        \begin{align*}
            \frac{1}{d}\sum^d_{h=1}\mathbb{E}_h[N_h]
            \leq\frac{Tb}{d}+\frac{T}{\sqrt{2}}\sqrt{\frac{9\epsilon^2}{5d}\cdot\sum^d_{h=1}\sum^T_{t=1}\mathbb{E}_u
            \left[\mathbb{I}_{h\in O_t}\right]}
            \leq\frac{Tb}{d}+\frac{3T\epsilon}{\sqrt{10}}\sqrt{\frac{Tb'}{d}}.
        \end{align*}
        Let $\epsilon=\frac{\sqrt{10}}{\beta}\cdot\sqrt{\frac{d}{b'T}}$ where $\beta\geq 12$.
        We have $\epsilon\leq\frac{1}{6\sqrt{d'}}\leq\frac{\rho}{3}$.
        Suppose $T\geq \frac{5}{2}d$ and $b\leq \alpha d$ where $0<\alpha<1-\frac{3}{\beta}$.
        We further obtain
        \begin{align}
            \frac{1}{d}\sum^d_{h=1}
            \frac{3}{5\sqrt{b}}\gamma^2\cdot\epsilon\cdot \left[T-\mathbb{E}_h\left[\sum^T_{t=1}\mathbb{I}_{h\in S_t}\right]\right]
            \geq&\left(1-\frac{3}{\beta}-\alpha\right)\cdot
            \frac{3\sqrt{10}\gamma^2}{5\beta\sqrt{b}}\cdot\sqrt{\frac{dT}{b'}}.
            \label{eq:ICLR2027:bounding_the_number_of_selection_optimal_attributes}
        \end{align}
    The regret of any deterministic algorithm for $(d,b,b_0)$-OSLR satisfies
    \begin{align*}
        &\sup_{({\bf x}_t,y_t),t\in[T]}\max_{{\bf w}\in\mathcal{W}_b}\mathrm{Reg}({\bf w})\\
        \geq&\sup_{h\in[d]}\sup_{({\bf x}^h_t,y_t),t\in[T]}\max_{{\bf w}\in\mathcal{W}_b}
        \left[\sum^T_{t=1}(\langle\mathbf{w}_t[S_t],\mathbf{x}^h_t[S_t]\rangle-y_t)^2
        -\sum^T_{t=1}(\langle\mathbf{w}[S],\mathbf{x}^h_t[S]\rangle-y_t)^2\right]\\
        =&\sup_{h\in[d]}\sup_{{\bf z}^h_t,t\in[T]}\max_{{\bf w}\in\mathcal{W}_b}
        \left[\sum^T_{t=1}\bar{\ell}(\mathbf{w}_t,S_t,\mathbf{z}^h_t)
        -\sum^T_{t=1}\bar{\ell}(\mathbf{w},S,\mathbf{z}^h_t)\right]\\
        \geq&\sup_{h\in[d]}\mathop{\mathbb{E}}_{{\bf z}^h_t,t\in[T]}
        \left[\sum^T_{t=1}\bar{\ell}(\mathbf{w}_t,S_t,\mathbf{z}^h_t)
        -\sum^T_{t=1}\bar{\ell}(\mathbf{w}_{\ast},S_{\ast},\mathbf{z}^h_t)\right]\\
        \geq&\sup_{h\in[d]}\mathop{\mathbb{E}}_{{\bf z}^h_t,t\in[T]}
        \left[\sum^T_{t=1}\bar{\ell}(\mathbf{w}_{S_t,\ast},S_t,\mathbf{z}^h_t)
        -\sum^T_{t=1}\bar{\ell}(\mathbf{w}_{\ast},S_{\ast},\mathbf{z}^h_t)\right]\\
        \geq&\frac{1}{d}\sum^d_{h=1}\mathop{\mathbb{E}}_{{\bf z}^h_t,t\in[T]}
        \left[\sum^T_{t=1}\bar{\ell}(\mathbf{w}_{S_t,\ast},S_t,\mathbf{z}^h_t)
        -\sum^T_{t=1}\bar{\ell}(\mathbf{w}_{\ast},S_{\ast},\mathbf{z}^h_t)\right]\\
        \mathop{\geq}^{(\ref{eq:ICLR2027:regret_to_number_of_selection_optimal_attributes})}&
        \frac{1}{d}\sum^d_{h=1}
        \cdot \frac{3}{5\sqrt{b}}\gamma^2\cdot\epsilon\cdot \left[T-\mathbb{E}_h\left[\sum^T_{t=1}\mathbb{I}_{h\in S_t}\right]\right]\\
        \mathop{\geq}^{(\ref{eq:ICLR2027:bounding_the_number_of_selection_optimal_attributes})}&
        \left(1-\frac{3}{\beta}-\alpha\right)\cdot
            \frac{3\sqrt{10}\gamma^2}{5\beta\sqrt{b}}\cdot\sqrt{\frac{dT}{b'}}.
    \end{align*}
    Finally,
    it holds that the lower bound for deterministic algorithm
    implies an identical lower bound for stochastic algorithm.

\end{proof}

\section{A Better Algorithm for $(d,b,b_0)$-OSLR with $b_0\geq 2$}

    The goal is to learn the optimal parameter ${\bf w}^\ast\in\mathcal{W}_b$.
    An intuitive approach is to construct $N$ subsets of $[d]$,
    each with size $b$, denoted by $S_i,i\in[N]$.
    Each element in $S_i$ indexes the corresponding attributes of ${\bf x}_t$.
    Let ${\bf w}^\ast_i\in\mathcal{W}_b$ be the optimal parameter with $\mathrm{Supp}({\bf w}^\ast_i)=S_i$.
    Thus ${\bf w}^\ast\in\{{\bf w}^\ast_1,\ldots,{\bf w}^\ast_N\}$
    and $\mathrm{Supp}({\bf w}_{\ast})\in\{S_1,\ldots,S_N\}$.
    The main idea of our algorithm is to
    simultaneously learn a probability distribution over the $N$ subsets,
    and ${\bf w}^\ast_i$ for all $i\in[N]$.
    Next we detail the algorithm.

    At round $t$,
    our algorithm maintains $N$ linear models $f_{t,1},f_{t,2},\ldots,f_{t,N}$
    and a probability distribution ${\bf p}_t\in \Delta_N$.
    For each $i\in[N]$,
    $f_{t,i}$ is parameterized by ${\bf w}_{t,i}\in\mathbb{R}^b$
    and uses the attributes indexed by $S_i$ to make a prediction,
    $
        f_{t,i}({\bf x}_t)=\langle {\bf w}_{t,i},{\bf x}_t[S_i]\rangle.
    $
    Our algorithm randomly selects a model $f_{t,I_t}$ following ${\bf p}_t$,
    observes the attributes indexed by $S_{I_t}$,
    and outputs $f_{t,I_t}({\bf x}_t)$.
    Since the attributes indexed by $S_i$, $i\neq I_t$, can not be fully observed,
    both the square loss of $f_{t,i}$ and its gradient denoted by $\nabla_{t,i}$ are unknown.
    It is necessary to estimate the loss and the gradient.

    For simplicity,
    let $c_{t,i}:=\left(f_{t,i}({\bf x}_t)-y_t\right)^2$.
    Unfolding the square loss function,
    $c_{t,i}$ and $\nabla_{t,i}$ can be rewritten as follows
    \begin{align*}
        c_{t,i}=&{\bf w}^\top_{t,i}\left({\bf x}_t{\bf x}^\top_t[S_i,S_i]\right){\bf w}_{t,i}
        -2y_t{\bf w}^\top_{t,i}{\bf x}_t[S_i]+y^2_t,\\
        \nabla_{t,i}=&2\left({\bf x}_t{\bf x}^\top_t[S_i,S_i]\right){\bf w}_{t,i}-2y_t{\bf x}_t[S_i].
    \end{align*}
    To estimate the loss and the gradient,
    it is necessary to estimate the covariance matrix ${\bf x}_t{\bf x}^\top_t$ and ${\bf x}_t$ itself.
    To this end,
    we uniformly sample $b_0$ attributes from
    $\{x_{t,1},\ldots,x_{t,d}\}\setminus \{x_{t,i}:i\in S_{I_t}\}$ without replacement.
    For simplicity,
    let $O_t=\{A_1,\ldots,A_{b_0}\}$ contain the indexes of the $b_0$ attributes,
    and $\mathbb{P}[i\in O_t]$ be probability that the $i$-th attribute is observed
    and $\mathbb{P}[i,j\in O_t]$
    be the probability that both the $i$-th attribute and the $j$-th attribute are observed.
    Let $\hat{\mathbf{x}}_t=(\hat{x}_{t,1},\hat{x}_{t,2},\ldots,\hat{x}_{t,d})^\top\in\mathbb{R}^d$
    and $\mathbf{h}_t\in\mathbb{R}^{d\times d}$ satisfy
    \begin{equation}
    \label{eq:IJTCS-FAW26:estimator_instance}
        \left\{
    \begin{split}
        \forall i\in S_{I_t},\quad \hat{x}_{t,i}=&x_{t,i},\\
        \forall i\in [d]\setminus S_{I_t},\quad\hat{x}_{t,i}=&\frac{x_{t,i}}{\mathbb{P}[i\in O_t]}\cdot\mathbb{I}_{i\in O_t},\\
        \forall i,j\in S_{I_t},\quad h_t[i,j]=&x_{t,i}x_{t,j},\\
        \forall i\in S_{I_t}, j\notin S_{I_t},\quad h_t[i,j]=&h[j,i]=\frac{x_{t,i}x_{t,j}}{\mathbb{P}[j\in O_t]}\cdot\mathbb{I}_{j\in O_t},\\
        \forall i\neq j\notin S_{I_t},\quad h_t[i,j]=&\frac{x_{t,i}x_{t,j}}{\mathbb{P}[i,j\in O_t]}\cdot\mathbb{I}_{i,j\in O_t},\\
        \forall i\notin S_{I_t},\quad h_t[i,i]=&\frac{x^2_{t,i}}{\mathbb{P}[i\in O_t]}\cdot\mathbb{I}_{i\in O_t}.
    \end{split}
    \right.
    \end{equation}
    It is easy to prove that $\hat{\mathbf{x}}_t$ and $\mathbf{h}_t$
    are unbiased estimator of $\mathbf{x}_t$ and $\mathbf{x}_t\mathbf{x}^\top_t$,
    respectively.
    The estimators of $c_{t,i}$ and $\nabla_{t,i}$ are defined by
    \begin{equation}
    \label{eq:AAAI25:loss_gradient_estimator}
    \left\{
    \begin{split}
        \tilde{c}_{t,i}
        =&{\bf w}^\top_{t,i}\mathbf{h}_t[S_i,S_i]{\bf w}_{t,i}
        -2y_t{\bf w}^\top_{t,i}\hat{\bf x}_t[S_i]+y^2_t,\\
        \tilde{\nabla}_{t,i}=&2\mathbf{h}_t[S_i,S_i]{\bf w}_{t,i}-2y_t\hat{\bf x}_t[S_i].
    \end{split}
    \right.
    \end{equation}
    We update the probabilities and parameters by online mirror descent (OMD),
    \begin{equation}
    \label{eq:IJTCS-FAW26:OMD_Negative_entropy}
    \left\{
    \begin{split}
        {\bf p}_{t+1}=&\mathop{\arg\min}_{{\bf p}\in\Delta_N}
        \left\langle\tilde{{\bf c}}_t,{\bf p}\right\rangle+\mathcal{B}_{\psi_t}({\bf p},{\bf p}_t),\qquad\qquad~
        \psi_t({\bf p})=\sum^N_{i=1}\frac{p_i}{\eta}\ln{p_i},\\
        {\bf w}_{t+1,i}=&\mathop{\arg\min}_{{\bf w}\in\mathcal{W}}
        \left\langle\tilde{\nabla}_{t,i},{\bf w}\right\rangle
        +\mathcal{B}_{\psi_{t,i}}({\bf w},{\bf w}_{t,i}),\quad
        \psi_{t,i}({\bf w})=\frac{\Vert {\bf w}\Vert^2_2}{2\lambda_{t,i}},
    \end{split}
    \right.
    \end{equation}
    in which $\psi({\bf p})$ is the negative entropy regularizer \cite{Bubeck2017Online}.
    It is easy to prove that
        $\forall i\notin S_{I_t}$,
        $\mathbb{P}[i\in O_t]=\frac{b_0}{d-b}$,
        and $\forall i\neq j\notin S_{I_t}$,
        $\mathbb{P}[i,j\in O_t]=\frac{b_0(b_0-1)}{(d-b)(d-b-1)}$.

    Finally, we give the solutions of \eqref{eq:IJTCS-FAW26:OMD_Negative_entropy}.
    \begin{equation}
    \label{eq:IJTCS-FAW26:updating_rule}
    \begin{split}
        p_{t+1,i}=&\frac{p_{t,i}\exp(-\eta\tilde{c}_{t,i})}{\sum^N_{j=1}p_{t,j}\exp(-\eta\tilde{c}_{t,j})},\quad
        i=1,2,\ldots,N,\\
        {\bf w}_{t+1,i}
        =&\min\left\{1,\frac{U}{\Vert {\bf w}_{t,i}-\lambda_{t,i}\tilde{\nabla}_{t,i}\Vert_2}\right\}\cdot
        \left({\bf w}_{t,i}-\lambda_{t,i}\tilde{\nabla}_{t,i}\right).
    \end{split}
    \end{equation}

    \begin{algorithm}[!t]
        \caption{Algorithm for $(d,b,b_0)$-OSLR with $b_0\geq 2$}
        \footnotesize
        \label{alg:IJTCS-FAW2026:>b+2}
        \begin{algorithmic}[1]
        \Require {$U, b, b_0, N$}
        \Ensure  $\{\mathbf{w}_{1,i}=\mathbf{0},p_{1,i}=\frac{1}{N}\}^N_{i=1}$\;
        \For{t=1,\ldots,T}
            \State Sample $I_t\in[N]$ following ${\bf p}_t$;
            \State Observe ${\bf x}_t[S_{I_t}]$;
            \State Output the prediction $\hat{y}_t=\langle {\bf w}_{t,I_t},{\bf x}_t[S_{I_t}]\rangle$;
            \State Sample $O_t\subseteq [d]\setminus\{x_{t,i}:i\in S_{I_t}\}$ without replacement;
            \State Construct $\hat{\mathbf{x}}_t$ and $\mathbf{h}_t$ by \eqref{eq:IJTCS-FAW26:estimator_instance};
            \State Compute $\tilde{c}_{t,i}$ and $\tilde{\nabla}_{t,i}$ by \eqref{eq:AAAI25:loss_gradient_estimator}
            \State Compute $\mathbf{p}_{t+1}$ and
            $\{\mathbf{w}_{t+1,i}\}^N_{i=1}$ by \eqref{eq:IJTCS-FAW26:updating_rule}
        \EndFor
        \end{algorithmic}
    \end{algorithm}

    Next we give the regret bound of Algorithm \ref{alg:IJTCS-FAW2026:>b+2}.

    \begin{theorem}
    \label{thm:IJTCS-FAW:regret_b>1}
        Let $q=\frac{b_0(b_0-1)}{(d-b)(d-b-1)}$ and
        \begin{align*}
            \lambda_{t,i}=\frac{U}{\sqrt{1+\sum^t_{\tau=1}\left\Vert\tilde{\nabla}_{\tau,i}\right\Vert^2_2}},\quad
            \eta=\sqrt{\frac{\ln{N}}{3 T\left(\frac{U^4X^4}{q}+4Y^2X^2U^2\cdot\frac{d-b}{b_0}+Y^4\right)}}.
        \end{align*}
        Suppose $T$ satisfies
        $$
            T\geq \frac{\left(\frac{U^2X^2}{q}+2YUX\frac{d-b}{b_0}+Y^2\right)^2\ln{N}}
            {3 \left(\frac{U^4X^4}{q}+4Y^2X^2U^2\cdot\frac{d-b}{b_0}+Y^4\right)}.
        $$
        The expected regret of Algorithm \ref{alg:IJTCS-FAW2026:>b+2} satisfies
        \begin{align*}
            &\forall {\bf w}\in\mathcal{W}_b,\quad\mathbb{E}\left[\mathrm{Reg}({\bf w})\right]\\
            &=O\left(\sqrt{T\left(\frac{U^4X^4}{q}+Y^2X^2U^2\cdot\frac{d-b}{b_0}+Y^4\right)\ln{N}}+
            UX\sqrt{\frac{U^2T}{q}+\frac{d-b}{b_0}Y^2T}\right).
        \end{align*}
    \end{theorem}

    Omitting the problem-independent constant factors,
    Algorithm \ref{alg:IJTCS-FAW2026:>b+2} achieves a $O(\sqrt{\frac{(d-b)(d-b-1)}{b_0(b_0-1)}}\sqrt{T\ln{N}})$ regret bound.
    In the case of $b_0\geq2$,
    previous algorithm \cite{Foster2016Online} enjoys a $O(\frac{d^2}{b^2_0}\sqrt{T\ln{N}})$ regret bound.
    Our algorithm improves their regret bound by a factor of $O(\frac{d}{b_0})$.
    Our algorithm is similar to the algorithm in \cite{Foster2016Online}.
    The main differences are the values of the two types learning rates
    $\eta$ and $\{\lambda_{t,i}\}^T_{t=1}$.
    To be specific,
    our learning rates are larger than their learning rates by a factor
    $O(\frac{d}{b_0})$.
    We prove that the expectation of $\tilde{c}^2_{t,i}$
    and $\Vert\tilde{\nabla}_{t,i}\Vert^2_2$ only scale with $O(\frac{1}{q})$,
    rather than $O(\frac{1}{q^2})$.
    Thus the learning rates can scale with $O(\sqrt{1/q})$.
    Another difference is that our estimators \eqref{eq:IJTCS-FAW26:estimator_instance}
    make use of the observed attributes $\mathbf{x}_t[S_{I_t}]$,
    while their estimators did not use the observed attributes.
    Thus our regret bound depends on $(d-b)(d-b-1)$ rather than $d^2$.

\section{A Better Algorithm for $(d,b,b_0)$-OSLR with $b_0\leq 1$}


    If we do not use the additional attribute (in the case of $b_0=1$),
    then selecting a subset from $\{S_i\}^N_{i=1}$ is indeed a $N$-armed bandit problem
    where the loss of each action is $c_{t,i}$.
    Suppose $S_{I_t}$ is selected at round $t$.
    If we just follow \eqref{eq:IJTCS-FAW26:updating_rule} with importance sampling estimators on the loss and gradient,
    then the variance of gradient estimator would scale with $\frac{1}{p_{t,i}}$, $i\in[N]$.
    The challenge is that it is hard to control $p_{t,i}$ which can approach to $0$.
    For instance, if $p_{t,i}=O(\frac{1}{T})$,
    then it is impossible to obtain a $o(T)$ regret bound.
    The same technical challenge arose in the problem of corralling a band of bandit algorithms
    \cite{Agarwal2017Corralling}
    and online kernel selection under bandit feedback \cite{Li2022Improved}.
    To address this issue,
    a customized master algorithm called $\mathrm{CORRAL}$ \cite{Agarwal2017Corralling}
    is established to carefully update $\mathbf{p}_t$.
    To be specific,
    $\mathrm{CORRAL}$ follows OMD with the regularizer
    $\psi_t({\bf p})=-\sum^N_{i=1}\frac{1}{\eta_t}\ln{p_i}$
    and non-decreasing learning rates $\eta_1\leq\eta_2\leq\ldots\leq\eta_T$.
    In this way,
    there is a negative term in the upper bound that scales with $\max_t\frac{1}{p_{t,i}}$,
    and can be used to cancel with the cumulated gradient variances.

    We use $\mathrm{CORRAL}$ (Algorithm \ref{alg:IJTCS-FAW2026:CORRAL}) as a black-box algorithm
    that receives $c_{1,I_1},\ldots,c_{t-1,I_{t-1}}$ as inputs
    and outputs $\mathbf{p}_t$.
    Then we sample $S_{I_t}$ where $I_t\in[N]$ follows $\mathbf{p}_t$
    and use $f_{t,I_t}$ to make a prediction
    $
        \hat{y}_t=f_{t,I_t}(\mathbf{x}_t)
    $.
    Now we define an unbiased gradient estimator by
    \begin{equation}
    \label{eq:IJTCS-FAW26:estimator_b_0=0}
        \forall i\in [N],\quad
        \tilde{\nabla}_{t,i}=\frac{\nabla_{t,i}}{p_{t,i}}\cdot\mathbb{I}_{i=I_t}.
    \end{equation}
    Note that there is no need to define $\tilde{c}_{t,i}$,
    which is constructed by CORRAL.
    We give the pseudo-code in Algorithm \ref{alg:IJTCS-FAW2026:=0}.

    \begin{algorithm}[!t]
        \caption{Algorithm for $(d,b,b_0)$-OSLR with $b_0\leq 1$}
        \footnotesize
        \label{alg:IJTCS-FAW2026:=0}
        \begin{algorithmic}[1]
        \Require {$U, b,  N$}
        \Ensure  $\{\mathbf{w}_{1,i}=\mathbf{0}\}^N_{i=1}$\;
        \For{$t=1,\ldots,T$}
            \State Obtain ${\bf p}_t$ by $\mathrm{CORRAL}(c_{1,I_1},\ldots,c_{t-1,I_{t-1}})$
            (Algorithm \ref{alg:IJTCS-FAW2026:CORRAL}).
            \State Sample $I_t\in[N]$ following ${\bf p}_t$ and observe ${\bf x}_t[S_{I_t}]$;
            \State Output the prediction $\hat{y}_t=\langle {\bf w}_{t,I_t},{\bf x}_t[S_{I_t}]\rangle$;
            \State Compute $\{\tilde{\nabla}_{t,i}\}^N_{i=1}$ by \eqref{eq:IJTCS-FAW26:estimator_b_0=0}
            \State Compute
            $\{\mathbf{w}_{t+1,i}\}^N_{i=1}$ by \eqref{eq:IJTCS-FAW26:updating_rule}
        \EndFor
        \end{algorithmic}
    \end{algorithm}

    \begin{algorithm}[!t]
        \caption{CORRAL \cite{Agarwal2017Corralling}}
        \footnotesize
        \label{alg:IJTCS-FAW2026:CORRAL}
        \begin{algorithmic}[1]
        \Require {$c_{1,I_1},\ldots,c_{t,I_{t}}$}
        \Ensure  $\gamma=\frac{1}{T},\beta=\mathrm{e}^{\frac{1}{\ln{T}}},\eta_{1,i}=\eta, \rho_{1,i}=2N$
        for all $i\in[N]$, $\mathbf{q}_1=\mathbf{p}_1=\frac{\mathbf{1}_N}{N}$\;
        \For{$t=1,\ldots,T$}
                \State $\forall i\in[N]$, $\tilde{c}_{t,i}=\frac{c_{t,i}}{p_{t,i}}\mathbb{I}_{i=I_t}$;
            \State Find $\lambda\in[\min_i\tilde{c}_{t,i},\max_i\tilde{c}_{t,i}]$ such that
                    $\sum^N_{i=1}\frac{1}{q^{-1}_{t,i}+\eta_{t,i}(\tilde{c}_{t,i}-\lambda)}=1$;
            \State Compute $\mathbf{q}_{t+1}$ such that $\frac{1}{q_{t+1,i}}=\frac{1}{q_{t,i}}+\eta_{t,i}(\tilde{c}_{t,i}-\lambda)$;
            \State Set $\mathbf{p}_{t+1}=(1-\gamma)\mathbf{q}_{t+1}+\gamma\frac{\mathbf{1}_N}{N}$;
            \For{$i=1,\ldots,N$}
                \If{$\frac{1}{p_{t+1,i}}>\rho_{t,i}$}
                    \State Set $\rho_{t+1,i}=\frac{2}{p_{t+1,i}}$, $\eta_{t+1,i}=\beta\eta_{t,i}$;
                \Else
                    \State Set $\rho_{t+1,i}=\rho_{t,i}$ and $\eta_{t+1,i}=\eta_{t.i}$;
                \EndIf
            \EndFor
        \EndFor
        \State Output: $\mathbf{p}_{t+1}$.
        \end{algorithmic}
    \end{algorithm}

    \begin{theorem}
    \label{thm:IJTCS-FAW:regret_0}
        Denote by $L=\max_{t,i}\Vert \nabla_{t,i}\Vert_2$ and
        $c=\max_{t,i}c_{t,i}$.
        Let
        $$
            \forall i\in[N],t\in[T],\quad
            \lambda_{t,i}=\frac{U}{\sqrt{1+\sum^t_{\tau=1}\left\Vert\tilde{\nabla}_{\tau,i}\right\Vert^2_2}},\qquad
            \eta=\frac{\sqrt{N}}{cUL\sqrt{T}}.
        $$
        The expected regret of Algorithm \ref{alg:IJTCS-FAW2026:=0} satisfies
        \begin{align*}
            \forall {\bf w}\in\mathcal{W}_b,\quad\mathbb{E}\left[\mathrm{Reg}({\bf w})\right]
            =O\left(cUL\sqrt{NT}\ln{T}\right).
        \end{align*}
    \end{theorem}

    Omitting the problem-independent constant factors,
    compared with the lower bound in Theorem \ref{thm:IJTSC-FAW26:lower_bound_regret:OSLR},
    Algorithm \ref{alg:IJTCS-FAW2026:=0} achieves a $O(\sqrt{NT}\ln{T})$ regret bound,
    and is optimal up to $O(\ln{T})$.
    In the case of $b_0\leq 1$,
    previous algorithm \cite{Ito2018Online} enjoys
    a $O\left(N^{\frac{1-q}{3-2q}}T^{\frac{2-q}{3-2q}}\right),q\in(0,1)$ regret bound.
    Algorithm \ref{alg:IJTCS-FAW2026:=0} improves their regret bound in terms of $T$.
    The $O(\ln{T})$ factor in the upper bound can be removed by using better black-box algorithms,
    such as the algorithm in \cite{Foster2020Adapting}.
    In this work,
    we use CORRAL as the black-box algorithm for the sake of simplicity.

\section{Discussion}

    Our work opens several directions for future research.
    The first one is to prove a stronger lower bound.
    We conjecture the $O(\frac{1}{\sqrt{b}})$ factor can be removed.
    The second one is to design better algorithms
    for the cases $b_0\geq 2$, $b_0=1$ and $b_0=0$.
    We did not design a customized algorithm for the case $b_0=1$.
    In other words,
    Algorithm \ref{alg:IJTCS-FAW2026:=0} did not use the additional attributes,
    and may be suboptimal.



%
%
%
 \bibliographystyle{splncs04}
 \bibliography{IJTCS-FAW26}
%
%
%
%
%

\section*{Appendix}

\section*{A. Preliminaries and Technical Lemma}

\subsection*{A~1. Bregman Divergence}

    If $\psi_t({\bf p})=\sum^N_{i=1}\frac{1}{\eta}p_i\ln p_i$,
    then the Bregman divergence is
    \begin{align}
        \mathcal{B}_{\psi_t}({\bf p},{\bf q})
        =&\psi_t({\bf p})-\psi_t({\bf q})-\langle \nabla\psi_t({\bf q}),{\bf p}-{\bf q}\rangle\nonumber\\
        =&\sum^N_{i=1}\frac{1}{\eta}p_i\ln p_i
        -\sum^N_{i=1}\frac{1}{\eta}q_i\ln q_i
        -\sum^N_{i=1}\frac{1}{\eta}(1+\ln{q_i})\cdot(p_i-q_{i})\nonumber\\
        =&\sum^N_{i=1}\frac{1}{\eta}p_i\ln\frac{p_i}{q_i}-\sum^N_{i=1}\frac{1}{\eta}(p_i-q_i).
        \label{eq:IJTCS-FAW26:bregman_div:negative_entropy}
    \end{align}
    If $\psi_t({\bf p})=\sum^N_{i=1}\frac{-\alpha}{\eta}p^{\frac{1}{\alpha}}_i$,
    then the Bregman divergence is
    \begin{align}
        \mathcal{B}_{\psi_t}({\bf p},{\bf q})
        =&\sum^N_{i=1}\frac{-\alpha}{\eta}p^{\frac{1}{\alpha}}_i
        -\sum^N_{i=1}\frac{-\alpha}{\eta}q^{\frac{1}{\alpha}}_{i}
        -\sum^N_{i=1}\frac{-1}{\eta}q^{\frac{1}{\alpha}-1}_{i}\cdot(p_i-q_{i})\nonumber\\
        =&-\sum^N_{i=1}\frac{\alpha}{\eta}p^{\frac{1}{\alpha}}_i
        +\sum^N_{i=1}\frac{\alpha-1}{\eta}q^{\frac{1}{\alpha}}_{i}
        +\sum^N_{i=1}\frac{1}{\eta}q^{\frac{1}{\alpha}-1}_{i}p_i.
        \label{eq:IJTCS-FAW26:bregman_div:tsalli}
    \end{align}
    If $\psi_t({\bf w})=\frac{1}{2\lambda_t}\Vert {\bf w}\Vert^2_2$,
    then the Bregman divergence is
    \begin{equation}
    \label{eq:IJTCS-FAW26:bregman_div:Euclidean_reg}
        \mathcal{B}_{\psi_t}({\bf w},{\bf v})
        =\psi_t({\bf w})-\psi_t({\bf v})-\langle \nabla\psi_t({\bf v}),{\bf w}-{\bf v}\rangle
        =\frac{1}{2\lambda_{t}}\Vert {\bf w}-{\bf v}\Vert^2_2.
    \end{equation}

\subsection*{A~2. First Optimality Condition}

    \begin{lemma}[First Optimality Condition \cite{Boyd2004Convex}]
    \label{lemma:ICML2024:first_order_optimization_condition}
        Assuming that $\Psi(\cdot):\Omega\rightarrow\mathbb{R}$
        is a convex and differential function,
        and $\Omega$ is a convex domain.
        Let $f^\ast=\mathrm{argmin}_{f\in\Omega}\Psi(f)$.
        Then
        \begin{align*}
            \forall g\in\Omega,\quad
            \langle \nabla\Psi(f^\ast),g-f^\ast\rangle \geq 0.
        \end{align*}
    \end{lemma}

    Next we show some properties of OMD that will be used in the subsequent analyses.
    ${\bf p}_{t+1}$ can be rewritten as follows
    \begin{align*}
        {\bf p}_{t+1}=&\mathop{\arg\min}_{{\bf p}\in\Delta_N}\Psi_t({\bf p}),\\
        \Psi_t({\bf p})=&\langle\tilde{{\bf c}}_t,{\bf p}\rangle+\mathcal{B}_{\psi_t}({\bf p},{\bf p}_t).
    \end{align*}
    Taking derivation w.r.t. ${\bf p}_{t+1}$ gives
    \begin{equation}
    \label{eq:AAAI25-supp:derivative:OMD}
        \nabla\Psi_t({\bf p}_{t+1})=\tilde{{\bf c}}_t+\nabla\psi_t({\bf p}_{t+1})-\nabla\psi_t({\bf p}_t).
    \end{equation}
    By Lemma \ref{lemma:ICML2024:first_order_optimization_condition}, we have
    \begin{align*}
        \forall {\bf u}\in\Delta_N,\quad \left\langle\nabla \Psi_t({\bf p}_{t+1}),{\bf p}_{t+1}-{\bf u}\right\rangle
        \leq 0.
    \end{align*}

\subsection*{A~3. Solutions of OMD}

    If $\psi_t({\bf p})=\sum^N_{i=1}\frac{1}{\eta}p_i\ln p_i$,
    then
    $$
        {\bf p}_{t+1}=\mathop{\arg\min}_{{\bf p}\in\Delta_{N}}
        \left(\left\langle\tilde{{\bf c}}_t,{\bf p}\right\rangle
        +\sum^N_{i=1}\frac{1}{\eta}p_i\ln\frac{p_i}{p_{t,i}}-\sum^N_{i=1}\frac{1}{\eta}(p_i-p_{t,i})\right).
    $$
    By the Lagrangian multiplier method,
    $p_{t+1,i}$ must satisfy
    \begin{align*}
        p_{t+1,i}=\frac{p_{t,i}\exp(-\eta\tilde{c}_{t,i})}{\sum^N_{j=1}p_{t,j}\exp(-\eta\tilde{c}_{t,j})}.
    \end{align*}
    If $\psi_t({\bf p})=\sum^N_{i=1}\frac{-\alpha}{\eta}p^{\frac{1}{\alpha}}_i$,
    then
    $$
        {\bf p}_{t+1}=\mathop{\arg\min}_{{\bf p}\in\Delta_{N}}
        \left(\left\langle\tilde{{\bf c}}_t,{\bf p}\right\rangle
        -\sum^N_{i=1}\frac{\alpha}{\eta}p^{\frac{1}{\alpha}}_i
        +\sum^N_{i=1}\frac{\alpha-1}{\eta}p^{\frac{1}{\alpha}}_{t,i}
        +\sum^N_{i=1}\frac{1}{\eta}p^{\frac{1}{\alpha}-1}_{t,i}p_i\right).
    $$
    By the Lagrangian multiplier method,
    $p_{t+1,i}$ must satisfy
    \begin{align*}
        p^{\frac{1}{\alpha}-1}_{t+1,i}-p^{\frac{1}{\alpha}-1}_{t,i}=\eta(\tilde{c}_{t,i}+\lambda).
    \end{align*}
    By the constraint $\sum^N_{i=1}p_{t+1,i}=1$.
    We can solve $\mathbf{p}_{t+1}$ by binary search.

    \noindent If $\psi_{t,i}({\bf w})=\frac{1}{2\lambda_{t,i}}\Vert {\bf w}\Vert^2_2$,
    then
    \begin{align*}
        \forall i\in[N],\quad{\bf w}_{t+1,i}=&\mathop{\arg\min}_{{\bf w}\in\mathcal{W}}
        \left\langle\tilde{\nabla}_{t,i},{\bf w}\right\rangle
        +\frac{1}{2\lambda_{t,i}}\Vert {\bf w}-{\bf w}_{t,i}\Vert^2_2.
    \end{align*}
    By the Lagrangian multiplier method, it is easy to obtain ${\bf w}_{t+1,i}$.
    \begin{align*}
        \forall i\in[N],\quad
        {\bf w}_{t+1,i}
        =&\min\left\{1,\frac{U}{\Vert\tilde{{\bf w}}_{t+1,i}\Vert_2}\right\}\cdot \tilde{{\bf w}}_{t+1,i},\\
        \tilde{{\bf w}}_{t+1,i}=&{\bf w}_{t,i}-\lambda_{t,i}\tilde{\nabla}_{t,i}.
    \end{align*}

\subsection*{A~4. Technical Lemma}

    \begin{lemma}
    \label{tab:IJTCS-FAW26:second_moment:>b+2}
        Let $q=\frac{b_0(b_0-1)}{(d-b)(d-b-1)}$.
        If $\tilde{c}_{t,i}$ and $\tilde{\nabla}_{t,i}$ follows \eqref{eq:AAAI25:loss_gradient_estimator},
        then
        \begin{align*}
            \vert \tilde{c}_{t,i}\vert \leq&\frac{1}{q}U^2X^2+2YUX\frac{d-b}{b_0}+Y^2,\\
            \mathbb{E}\left[\tilde{c}^2_{t,i}\right]\leq&
            3\left(\frac{1}{q}U^4X^4+4Y^2X^2U^2\frac{d-b}{b_0}+Y^4\right),\\
            \mathbb{E}\left[\Vert\tilde{\nabla}^2_{t,i}\Vert^2_2\right]\leq& \frac{8}{q}U^2X^2+8Y^2X^2\frac{d-b}{b_0}.
        \end{align*}
    \end{lemma}
    \begin{proof}
        By \eqref{eq:IJTCS-FAW26:estimator_instance},
        we have
        \begin{align*}
            \vert \tilde{c}_{t,i}\vert
            =&\left\vert {\bf w}^\top_{t,i}\mathbf{h}_t[S_i,S_i]{\bf w}_{t,i}
            -2y_t{\bf w}^\top_{t,i}\hat{\bf x}_t[S_i]+y^2_t\right\vert\\
            \leq&U^2\left\Vert\mathbf{h}_t[S_i,S_i]\right\Vert_{\mathrm{F}}
            +2YU\Vert\hat{\bf x}_t[S_i]\Vert_2+Y^2\\
            =&U^2\sqrt{\sum_{m,n\in S_i}h^2_t[m,n]}
            +2YU\sqrt{\sum_{m\in S_i}\hat{x}^2_{t,i}}+Y^2\\
            \leq&U^2\sqrt{\frac{1}{q^2}\sum_{m,n\in S_i}x^2_{t,m}x^2_{t,n}}
            +2YU\sqrt{\sum_{m\in S_i}\frac{(d-b)^2}{b^2_0}x^2_{t,m}}+Y^2\\
            \leq&\frac{1}{q}U^2X^2+2YUX\frac{d-b}{b_0}+Y^2.
        \end{align*}
        Next we analyze the second-order central moment.
        \begin{align*}
            \mathbb{E}[\tilde{c}^2_{t,i}]
            =&\mathbb{E}\left[\left({\bf w}^\top_{t,i}\mathbf{h}_t[S_i,S_i]{\bf w}_{t,i}
            -2y_t{\bf w}^\top_{t,i}\hat{\bf x}_t[S_i]+y^2_t\right)^2\right]\\
            \leq&3\mathbb{E}\left[\left({\bf w}^\top_{t,i}\mathbf{h}_t[S_i,S_i]{\bf w}_{t,i}\right)^2+
            \left(2y_t{\bf w}^\top_{t,i}\hat{\bf x}_t[S_i]\right)^2+y^4_t\right]\\
            \leq&3\mathbb{E}\left[\Vert{\bf w}_{t,i}\Vert^4_2\cdot\Vert\mathbf{h}_t[S_i,S_i]\Vert^2_{\mathrm{F}}+
            4y^2_t\left\Vert{\bf w}_{t,i}\Vert^2_2\cdot\Vert\hat{\bf x}_t[S_i]\right\Vert^2_2+y^4_t\right]\\
            \leq&3\left(U^4\cdot\sum_{m,n\in S_i}\mathbb{E}[ h^2_t[m,n]]+
            4Y^2U^2\cdot\sum_{m\in S_i}\mathbb{E}[\hat{x}^2_{t,m}]+Y^4\right)\\
            \leq&3\left(\frac{U^4}{q}\cdot\Vert (\mathbf{x}_t\mathbf{x}^\top_t)[S_i,S_t]\Vert^2_F+
            4Y^2U^2\cdot\frac{d-b}{b_0}\cdot \Vert\mathbf{x}_t[S_i]\Vert^2_2+Y^4\right)\\
            \leq&3\left(\frac{1}{q}U^4X^4+4Y^2X^2U^2\frac{d-b}{b_0}+Y^4\right),
        \end{align*}
        in which
        \begin{align*}
            \mathbb{E}[h^2_t[m,n]]=&x^2_{t,m}x^2_{t,n}\cdot\mathbb{I}_{m,n\in S_{I_t}}
            +x^2_{t,m}\frac{x^2_{t,n}}{\mathbb{P}[n\in O_t]}\cdot\mathbb{I}_{m\in S_{I_t}, n\notin S_{I_t}}
            +\\
            &x^2_{t,n}\frac{x^2_{t,m}}{\mathbb{P}[m\in O_t]}\cdot\mathbb{I}_{m\notin S_{I_t}, n\in S_{I_t}}+
            \frac{x^2_{t,m}x^2_{t,n}}{\mathbb{P}[m,n\in O_t]}\cdot\mathbb{I}_{m,n\notin S_{I_t}}\\
            \leq&\frac{x^2_{t,m}x^2_{t,n}}{\mathbb{P}[m,n\in O_t]}\\
            \leq&\frac{x^2_{t,m}x^2_{t,n}}{q},
        \end{align*}
        and
        \begin{align*}
            \mathbb{E}[\hat{x}^2_{t,m}]
            =x^2_{t,m}\cdot\mathbb{I}_{m\in S_{I_t}}
            +\frac{x^2_{t,m}}{\mathbb{P}[m\in O_t]}\cdot\mathbb{I}_{m\notin S_{I_t}}
            \leq\frac{x^2_{t,m}}{\mathbb{P}[m\in O_t]}
            \leq \frac{d-b}{b_0}x^2_{t,m}.
        \end{align*}
        For the second inequality,
        \begin{align*}
            \mathbb{E}\left[\left\Vert\tilde{\nabla}_{t,i}\right\Vert^2_2\right]
            =&\mathbb{E}\left[\Vert 2\mathbf{h}_t[S_i,S_i]{\bf w}_{t,i}-2y_t\hat{\bf x}_t[S_i]\Vert^2_2\right]\\
            \leq&\mathbb{E}\left[2\Vert 2\mathbf{h}_t[S_i,S_i]{\bf w}_{t,i}\Vert^2_2+
            2\Vert2y_t\hat{\bf x}_t[S_i]\Vert^2_2\right]\\
            \leq&8U^2\mathbb{E}\left[\Vert \mathbf{h}_t[S_i,S_i]\Vert^2_{\mathrm{F}}\right]+
            8Y^2\mathbb{E}\left[\Vert \hat{\bf x}_t[S_i]\Vert^2_2\right]\\
            \leq&\frac{8}{q}U^2X^2+8Y^2X^2\frac{d-b}{b_0}.
        \end{align*}
        We conclude the proof.
    \end{proof}

\section*{B. Proof of Theorem \ref{thm:IJTCS-FAW:regret_b>1}}

    \begin{proof}
        For any ${\bf w}'\in \mathcal{W}_b$, there exists a ${\bf w}\in\mathcal{W}$ and $i\in[N]$
        such that
        \begin{align*}
            \sum^T_{t=1}(\langle{\bf w}',{\bf x}_t\rangle -y_t)^2
            =\sum^T_{t=1}(\langle{\bf w},{\bf x}_t[S_i]\rangle -y_t)^2.
        \end{align*}
        Without loss of generality,
        suppose that $\mathrm{Supp}({\bf w}')$ is a subset of $S_i$.
        We decompose the regret as follows
        \begin{align*}
            \mathrm{Reg}({\bf w})
            =&\sum^T_{t=1}(\hat{y}_t-y_t)^2-\sum^T_{t=1}(\langle{\bf w},{\bf x}_t[S_i]\rangle -y_t)^2\\
            =&\sum^T_{t=1}\left[(\langle{\bf w}_{t,I_{t,1}},{\bf x}_t[S_{I_{t,1}}]\rangle-y_t)^2-(\langle{\bf w}_{t,i},{\bf x}_t[S_i]\rangle -y_t)^2\right]+\\
            &\sum^T_{t=1}\left[(\langle{\bf w}_{t,i},{\bf x}_t[S_i]\rangle -y_t)^2-(\langle{\bf w},{\bf x}_t[S_i]\rangle -y_t)^2\right]\\
            =&\underbrace{\sum^T_{t=1}\left[c_{t,I_{t,1}}-c_{t,i}\right]}_{R_1}+
            \underbrace{\sum^T_{t=1}\left[(\langle{\bf w}_{t,i},{\bf x}_t[S_i]\rangle -y_t)^2-(\langle{\bf w},{\bf x}_t[S_i]\rangle -y_t)^2\right]}_{R_2}.
        \end{align*}
        Next we separately analyze $R_1$ and $R_2$.

\subsection*{B~1. Analyzing $R_1$}

        Let $\mathbf{u}\in \Delta_N$ and $u_i=1$.
        We analyze the instantaneous regret.
        \begin{align*}
            \langle {\bf p}_t-{\bf u},{\bf c}_t\rangle
            =&\langle {\bf p}_t-{\bf u},\tilde{{\bf c}}_t\rangle
            +\langle {\bf p}_t-{\bf u},{\bf c}_t-\tilde{{\bf c}}_t\rangle\\
            =&\underbrace{\langle {\bf p}_{t+1}-{\bf u},\tilde{{\bf c}}_t\rangle}_{\Xi_{1,1}}
            +\underbrace{\langle {\bf p}_t-{\bf p}_{t+1},\tilde{{\bf c}}_t\rangle}_{\Xi_{1,2}}
            +\langle {\bf p}_t-{\bf u},{\bf c}_t-\tilde{{\bf c}}_t\rangle.
        \end{align*}
        Next we give an upper bound on $\Xi_{1,1}$.
        \begin{align*}
            &\mathcal{B}_{\psi_t}({\bf u},{\bf p}_t)
            -\mathcal{B}_{\psi_t}({\bf u},{\bf p}_{t+1})
            -\mathcal{B}_{\psi_t}({\bf p}_{t+1},{\bf p}_t)\\
            =&\langle\nabla\psi_t({\bf p}_t) -\nabla\psi_t({\bf p}_{t+1}),{\bf p}_{t+1}-{\bf u}\rangle\\
            =&\langle\tilde{{\bf c}}_t+\nabla\psi_t({\bf p}_t) -\nabla\psi_t({\bf p}_{t+1})
            -\tilde{{\bf c}}_t,{\bf p}_{t+1}-{\bf u}\rangle\\
            \mathop{=}^{\eqref{eq:AAAI25-supp:derivative:OMD}}&
            \langle\tilde{{\bf c}}_t,{\bf p}_{t+1}-{\bf u}\rangle
            -\langle\nabla \Psi_t({\bf p}_{t+1}),{\bf p}_{t+1}-{\bf u}\rangle\\
            \geq&\langle\tilde{{\bf c}}_t,{\bf p}_{t+1}-{\bf u}\rangle,
            \quad(\mathrm{Lemma}~\ref{lemma:ICML2024:first_order_optimization_condition})
        \end{align*}
        where ${\bf p}_{t+1}\in\Delta_{N}$ and ${\bf u}\in\Delta_N$.
        Similarly,
        $\Xi_{1,2}$ can be rewritten by
        \begin{align*}
            \mathcal{B}_{\psi_t}({\bf p}_t,{\bf p}_t)
            -\mathcal{B}_{\psi_t}({\bf p}_t,{\bf p}_{t+1})
            -\mathcal{B}_{\psi_t}({\bf p}_{t+1},{\bf p}_t)
            =&\langle\nabla\psi_t({\bf p}_t) -\nabla\psi_t({\bf p}_{t+1}),{\bf p}_{t+1}-{\bf p}_t\rangle\\
            =&\langle \tilde{{\bf c}}_t,{\bf p}_{t+1}-{\bf p}_t\rangle.
        \end{align*}
        Taking expectation w.r.t. $\{O_t\}^T_{t=1}$ gives
        \begin{align*}
            &\mathbb{E}\left[\sum^T_{t=1}\langle {\bf p}_t-{\bf u},{\bf c}_t\rangle\right]\\
            \leq&\mathbb{E}\left[\sum^T_{t=1}\left[\mathcal{B}_{\psi_t}({\bf u},{\bf p}_t)
            -\mathcal{B}_{\psi_t}({\bf u},{\bf p}_{t+1})\right]+
            \sum^T_{t=1}\mathcal{B}_{\psi_t}({\bf p}_t,{\bf p}_{t+1})+
            \sum^T_{t=1}\langle {\bf p}_t-{\bf u},{\bf c}_t-\tilde{{\bf c}}_t\rangle\right]\\
            \leq&\mathcal{B}_{\psi_t}({\bf u},{\bf p}_1)
            +\mathbb{E}\left[\sum^T_{t=1}\mathcal{B}_{\psi_t}({\bf p}_t,{\bf p}_{t+1})\right],
        \end{align*}
        in which $\mathbb{E}_t[\tilde{{\bf c}}_t]={\bf c}_t$ and $\mathbf{p}_1=\frac{1}{N}\mathbf{1}_N$.
        We further analyze the second term.
        Let
        \begin{align*}
            \nabla\psi_t(\tilde{{\bf p}}_{t+1})=\nabla\psi_t({\bf p}_t)-\tilde{\bf c}_t.
        \end{align*}
        which gives
        \begin{equation}
        \label{eq:IJTCS-FAW26:OMD_updating_rule:entropy}
            \forall i=1,\ldots,N,\quad
            \ln{\tilde{p}_{t+1,i}}=\ln{(p_{t,i})}-\eta\tilde{c}_{t,i}.
        \end{equation}
        Then ${\bf p}_{t+1}$ can be equivalently defined as follows
        \begin{align*}
            {\bf p}_{t+1}=\mathop{\arg\min}_{{\bf p}\in\Delta_N}\mathcal{B}_{\psi_t}({\bf p},\tilde{{\bf p}}_{t+1}).
        \end{align*}
        It use the inequality
        $\exp(-x)\leq 1-x+x^2$ for $x\in[-1,1]$.
        Thus we can obtain
        \begin{align*}
            &\mathbb{E}\left[\sum^T_{t=1}\mathcal{B}_{\psi_t}({\bf p}_t,{\bf p}_{t+1})\right]
            \leq \mathbb{E}\left[\sum^T_{t=1}\mathcal{B}_{\psi_t}({\bf p}_t,\tilde{\bf p}_{t+1})\right]\\
            \mathop{=}^{\eqref{eq:IJTCS-FAW26:bregman_div:tsalli}}&
            \mathbb{E}\left[\sum^T_{t=1}\frac{1}{\eta}\left(
        \sum^N_{i=1}p_{t,i}\ln\frac{p_{t,i}}{\tilde{p}_{t+1,i}}-\sum^N_{i=1}(p_{t,i}-\tilde{p}_{t+1,i})\right)\right]\\
            \mathop{=}^{\eqref{eq:IJTCS-FAW26:OMD_updating_rule:entropy}}&
            \mathbb{E}\left[\sum^T_{t=1}\frac{1}{\eta}\left(
        \sum^N_{i=1}\eta p_{t,i}\tilde{c}_{t,i}-\sum^N_{i=1}(p_{t,i}-p_{t,i}\exp(-\eta\tilde{c}_{t,i}))\right)\right]\\
            \leq& \mathbb{E}\left[\sum^T_{t=1}\frac{1}{\eta}\left(
        \sum^N_{i=1}\eta p_{t,i}\tilde{c}_{t,i}+\sum^N_{i=1}p_{t,i}\left(1-\eta\tilde{c}_{t,i}+\eta^2\tilde{c}^2_{t,i}\right)-1\right)\right]\\
            =&\mathbb{E}\left[\sum^T_{t=1}\eta\sum^N_{i=1}p_{t,i}\tilde{c}^2_{t,i}\right]\\
        \leq&3\eta T\left(\frac{1}{q}U^4X^4+4Y^2X^2U^2\cdot\frac{d-b}{b_0}+Y^4\right),
        \qquad\mathrm{by}~\mathrm{Lemma}~\ref{tab:IJTCS-FAW26:second_moment:>b+2}
        \end{align*}
        where we have $\eta\max\vert\tilde{c}_{t,i}\vert\leq 1$ and
        Combining the above results,
        we have
        \begin{align*}
            \mathbb{E}[R_1]=&\mathbb{E}\left[\sum^T_{t=1}\langle{\bf p}_t-{\bf u},{\bf c}_t\rangle\right]\\
            \leq&\frac{1}{\eta}\ln{N}+
            3\eta T\left(\frac{1}{q}U^4X^4+4Y^2X^2U^2\cdot\frac{d-b}{b_0}+Y^4\right)\\
            \leq&2\sqrt{3 T\left(\frac{1}{q}U^4X^4+4Y^2X^2U^2\cdot\frac{d-b}{b_0}+Y^4\right)\ln{N}}.
        \end{align*}

\subsection*{B~2. Analyzing $R_2$}

        The analysis is similar to the analysis of $R_1$.
        \begin{align*}
            \mathbb{E}\left[R_2\right]
            \leq&\mathbb{E}\left[\sum^T_{t=1}\langle \nabla_{t,i},{\bf w}_{t,i}-{\bf w}\rangle\right]\\
            =&\mathbb{E}\left[\sum^T_{t=1}\left\langle \tilde{\nabla}_{t,i},{\bf w}_{t,i}-{\bf w}\right\rangle
            +\sum^T_{t=1}\left\langle {\bf w}_{t,i}-{\bf w},\nabla_{t,i}-\tilde{\nabla}_{t,i}\right\rangle\right]\\
            \leq&\mathbb{E}\left[\sum^T_{t=1}\left[\mathcal{B}_{\psi_{t,i}}({\bf w},{\bf w}_{t,i})
            -\mathcal{B}_{\psi_{t,i}}({\bf w},{\bf w}_{t+1,i})\right]+
            \sum^T_{t=1}\mathcal{B}_{\psi_{t,i}}({\bf w}_{t,i},{\bf w}_{t+1,i})\right],
        \end{align*}
        where
        $$
            \mathbb{E}\left[\sum^T_{t=1}\left\langle {\bf w}_{t,i}-{\bf w},\nabla_{t,i}-\tilde{\nabla}_{t,i}\right\rangle\right]
            =\sum^T_{t=1}\mathbb{E}_t\left[\left\langle {\bf w}_{t,i}-{\bf w},\nabla_{t,i}-\tilde{\nabla}_{t,i}\right\rangle\right]
            =0.
        $$
        Recalling that
        \begin{align*}
            \lambda_{t,i}=\frac{U}{\sqrt{1+\sum^t_{\tau=1}\left\Vert\tilde{\nabla}_{\tau,i}\right\Vert^2_2}}.
        \end{align*}
        By \eqref{eq:IJTCS-FAW26:bregman_div:Euclidean_reg},
        we have
        \begin{align*}
            \sum^T_{t=1}\mathcal{B}_{\psi_{t,i}}({\bf w}_{t,i},{\bf w}_{t+1,i})
            \leq\sum^T_{t=1}\frac{\Vert {\bf w}_{t,i}-\tilde{\bf w}_{t+1,i}\Vert^2_2}{2\lambda_{t,i}}
            \leq\sum^T_{t=1}\frac{\lambda_{t,i}}{2}\left\Vert \tilde{\nabla}_{t,i}\right\Vert^2_2
            \leq U\sqrt{\sum^T_{t=1}\Vert \tilde{\nabla}_{t,i}\Vert^2_2},
        \end{align*}
        and
        \begin{align*}
            &\sum^T_{t=1}[\mathcal{B}_{\psi_{t,i}}({\bf w},{\bf w}_{t,i})
            -\mathcal{B}_{\psi_{t,i}}({\bf w},{\bf w}_{t+1,i})]\\
            =&\sum^T_{t=1}\frac{1}{2\lambda_{t,i}}\Vert {\bf w}-{\bf w}_{t,i}\Vert^2_2
            -\frac{1}{2\lambda_{t,i}}\Vert {\bf w}-{\bf w}_{t+1,i}\Vert^2_2\\
            \leq&\frac{1}{2}\sum^{T-1}_{t=1}\Vert {\bf w}-{\bf w}_{t+1,i}\Vert^2_2\cdot\left(\frac{1}{\lambda_{t+1,i}}
            -\frac{1}{\lambda_{t,i}}\right)+\frac{\Vert {\bf w}-{\bf w}_{1,i}\Vert^2_2}{2\lambda_1}\\
            \leq&\frac{U^2}{\lambda_{T,i}}
            =U\sqrt{1+\sum^T_{\tau=1}\left\Vert\tilde{\nabla}_{\tau,i}\right\Vert^2_2}.
        \end{align*}
        By Lemma \ref{tab:IJTCS-FAW26:second_moment:>b+2},
        taking expectation w.r.t. $\{O_t\}^T_{t=1}$ gives
        \begin{align*}
            \mathbb{E}[R_2]
            \leq2U\sqrt{\sum^T_{t=1}\mathbb{E}\left[\left\Vert \tilde{\nabla}_{t,i}\right\Vert^2_2\right]}+U
            \leq2U\sqrt{\frac{8U^2X^2T}{q}+\frac{d-b}{b_0}8Y^2X^2T}+U,
        \end{align*}
        where the last inequality comes from
        Lemma \ref{tab:IJTCS-FAW26:second_moment:>b+2}.

        Combining the upper bound on $\mathbb{E}[R_1]$ and $\mathbb{E}[R_2]$ gives
        \begin{align*}
            &\mathbb{E}[\mathrm{Reg}({\bf w})]\\
            \leq&\mathbb{E}[R_1]+\mathbb{E}[R_2]\\
            \leq&2\sqrt{3 T\left(\frac{U^4X^4}{q}+4Y^2X^2U^2\frac{d-b}{b_0}+Y^4\right)\ln{N}}+
            2U\sqrt{\frac{8U^2X^2T}{q}+8Y^2X^2\frac{d-b}{b_0}T}+U.
        \end{align*}

    \end{proof}

\section*{C. Proof of Theorem \ref{thm:IJTCS-FAW:regret_0}}

    By the regret of CORRAL,
    we can directly give an upper bound $R_1$.

    \begin{lemma}[Lemma 13 \cite{Agarwal2017Corralling}]
        Suppose $\max_{t,i}c_{t,i}=c$.
        Let $\rho_{T,i}$ follow Algorithm \ref{alg:IJTCS-FAW2026:CORRAL}.
        Then CORRAL ensures that
        $$
            R_1=O\left(\frac{N\ln{T}}{\eta}+c^2T\eta\right)+\sum^T_{t=1}\frac{2c_{t,I_t}}{TN}-\frac{\rho_{T,i}}{40\eta \ln{T}}.
        $$
    \end{lemma}

    Next we just need to analyze $R_2$.
    \begin{proof}
        Let $\max_{t,i}\Vert \nabla_{t,i}\Vert_2=:L$.
        \begin{align*}
            R_2\leq2U\sqrt{\sum^T_{t=1}\left\Vert \tilde{\nabla}_{t,i}\right\Vert^2_2}+U
            =&2U\sqrt{\sum^T_{t=1}\frac{\left\Vert \nabla_{t,i}\right\Vert^2_2}{p^2_{t,i}}\cdot\mathbb{I}_{I_t=i}}+U\\
            \leq&2UL\sqrt{\sum^T_{t=1}\max_{t}\frac{1}{p_{t,i}}\cdot\frac{\mathbb{I}_{I_t=i}}{p_{t,i}}}+U\\
            \leq&2UL\sqrt{\rho_{T,i}\cdot\sum^T_{t=1}\frac{\mathbb{I}_{I_t=i}}{p_{t,i}}}+U.
        \end{align*}
        Then the regret is
        \begin{align*}
            &\mathrm{Reg}({\bf w})\\
            =&R_1+R_2\\
            =& O\left(\frac{N\ln{T}}{\eta}+c^2T\eta\right)+\sum^T_{t=1}\frac{2c_{t,I_t}}{TN}-\frac{\rho_{T,i}}{40\eta \ln{T}}
            +2UL\sqrt{\rho_{T,i}\cdot\sum^T_{t=1}\frac{\mathbb{I}_{I_t=i}}{p_{t,i}}}+U\\
            =& O\left(\frac{N\ln{T}}{\eta}+c^2T\eta\right)+\sum^T_{t=1}\frac{2c_{t,I_t}}{TN}-\frac{\rho_{T,i}}{40\eta \ln{T}}
            +2UL\sqrt{\rho_{T,i}\cdot\sum^T_{t=1}\frac{\mathbb{I}_{I_t=i}}{p_{t,i}}}+U\\
            \leq& O\left(\frac{N\ln{T}}{\eta}+c^2T\eta\right)+\frac{2c}{N}
            +40\eta U^2L^2\cdot \sum^T_{t=1}\frac{\mathbb{I}_{I_t=i}}{p_{t,i}}\cdot \ln{T}+U.
        \end{align*}
        Let $\eta=\frac{\sqrt{N}}{cUL\sqrt{T}}$ .
        Taking expectation w.r.t. $\{I_t\}^T_{t=1}$ yields
        \begin{align*}
            \mathbb{E}\left[\mathrm{Reg}({\bf w})\right]
            =& O\left(\frac{N\ln{T}}{\eta}+c^2T\eta\right)+40\eta U^2L^2T\cdot \ln{T}+U\\
            =& O\left(cUL\sqrt{NT}\ln{T}\right),
        \end{align*}
        which concludes the proof.
    \end{proof}

\end{document}